\documentclass[sigconf]{acmart}
\AtBeginDocument{%
  \providecommand\BibTeX{{%
    \normalfont B\kern-0.5em{\scshape i\kern-0.25em b}\kern-0.8em\TeX}}}

\copyrightyear{2022}
\acmYear{2022}
\setcopyright{acmcopyright}\acmConference[TheWebConf '22]{Proceedings of the ACM Web Conference 2022}{April 25--29, 2022}{Lyon, France.}
\acmBooktitle{Proceedings of the ACM Web Conference 2022, April 25--29, 2022, Lyon, France.}
\acmPrice{15.00}
\acmDOI{10.1145/3437963.3441811}
\acmISBN{978-1-4503-8297-7/21/03}

\usepackage{hyperref}

\usepackage{thmtools,thm-restate}
\usepackage{soul}
\usepackage{url}
\usepackage{graphicx}
\usepackage{amsmath}
\usepackage{amsthm}
\usepackage{booktabs}
\usepackage{algorithm}
\usepackage{algorithmic}

\usepackage{microtype}
\usepackage{subcaption}
\usepackage{multirow}
\usepackage{tabularx}

\usepackage{amsmath,amsfonts,bm}

\def\eqref#1{equation~\ref{#1}}

\def\1{\bm{1}}

\def\rve{{\mathbf{e}}}

\def\rvr{{\mathbf{r}}}

\DeclareMathAlphabet{\mathsfit}{\encodingdefault}{\sfdefault}{m}{sl}
\SetMathAlphabet{\mathsfit}{bold}{\encodingdefault}{\sfdefault}{bx}{n}

\def\gE{{\mathcal{E}}}

\def\gG{{\mathcal{G}}}

\def\gM{{\mathcal{M}}}

\def\gR{{\mathcal{R}}}
\def\gS{{\mathcal{S}}}
\def\gT{{\mathcal{T}}}

\def\sA{{\mathbb{A}}}

\newcommand{\E}{\mathbb{E}}

\newtheorem{definition}{Definition}

\newtheorem{proposition}{Proposition}
\newtheorem{propositionEnd}{Proposition}
\newtheorem{example}{Example}

\begin{document}
\title{Principled Representation Learning for Entity Alignment}

\author{
	Lingbing Guo$^{1}$, Zequn Sun$^{2}$, Mingyang Chen$^{1}$, Wei Hu$^{2}$, Qiang Zhang$^{1}$, Huajun Chen$^{1}$
}
\affiliation{
	\institution{$^1$ College of Computer Science and Technology, Zhejiang University, Hangzhou, China} \country{}
}

\affiliation{
	\institution{$^2$ State Key Laboratory for Novel Software Technology, Nanjing University, Nanjing, China} \country{}
}

\email{
	{lbguo, mingyangchen, qiang.zhang.cs, huajunsir}@zju.edu.cn
}
\email{
	zqsun.nju@gmail.com, whu@nju.edu.cn
}

\begin{abstract}
Embedding-based entity alignment (EEA) has recently received great attention. Despite significant performance improvement, few efforts have been paid to facilitate understanding of EEA methods. Most existing studies rest on the assumption that a small number of pre-aligned entities can serve as anchors connecting the embedding spaces of two KGs. Nevertheless, no one investigates the rationality of such an assumption. To fill the research gap, we define a typical paradigm abstracted from existing EEA methods and analyze how the embedding discrepancy between two potentially aligned entities is implicitly bounded by a predefined margin in the scoring function. Further, we find that such a bound cannot guarantee to be tight enough for alignment learning. We mitigate this problem by proposing a new approach, named NeoEA, to explicitly learn KG-invariant and principled entity embeddings. In this sense, an EEA model not only pursues the closeness of aligned entities based on geometric distance, but also aligns the neural ontologies of two KGs by eliminating the discrepancy in embedding distribution and underlying ontology knowledge. Our experiments demonstrate consistent and significant improvement in performance against the best-performing EEA methods.
\end{abstract}


\maketitle

\section{Introduction}
\label{sec:intro}
Knowledge graphs (KGs), such as DBpedia~\cite{DBpedia} and Wikidata~\cite{Wikidata}, have become crucial resources for many AI applications. Although a large-scale KG offers structured knowledge derived from millions of facts in the real world, it is still incomplete by nature, and the downstream applications are always demanding more knowledge. Entity alignment (EA) is then proposed to solve this issue, which exploits the potentially aligned entities among different KGs to facilitate knowledge fusion and exchange.

Recently, embedding-based entity alignment (EEA) methods~\cite{MTransE,JAPE,IPTransE,GCN-Align,RSN,AVR-GCN,RDGCN,AliNet,EA_iclr} have been prevailing in the EA area. The central idea is to encode entity/relation semantics into embeddings and estimate entity similarity based on their embedding distance. These methods either learn an alignment function $f_a$ to minimize the distance between the embeddings of a pair of aligned entities \cite{GCN-Align}, or directly map these two entities to one vector representation~\cite{JAPE}. Meanwhile, they leverage a shared scoring function $f_s$ to encode semantics into representations, such that two potentially aligned entities 
shall have similar feature expression. During this process, a small number of aligned entity pairs (a.k.a., seed alignment) are required as supervision data to align (or merge) the embedding spaces of two KGs.

Different from the conventional laborious methods~\cite{PARIS,ALEX,SiGMa} which manually collect and select discriminative features, EEA ones rarely rely on third-party tools or pipeline for preprocessing. Commonly, they just need triples or adjacency matrices of KGs as the input data, and have achieved comparable or better performance on many benchmarks~\cite{JAPE,BootEA,OpenEA}.
\begin{figure}[t]
	\centering
	\includegraphics[width=.8\linewidth]{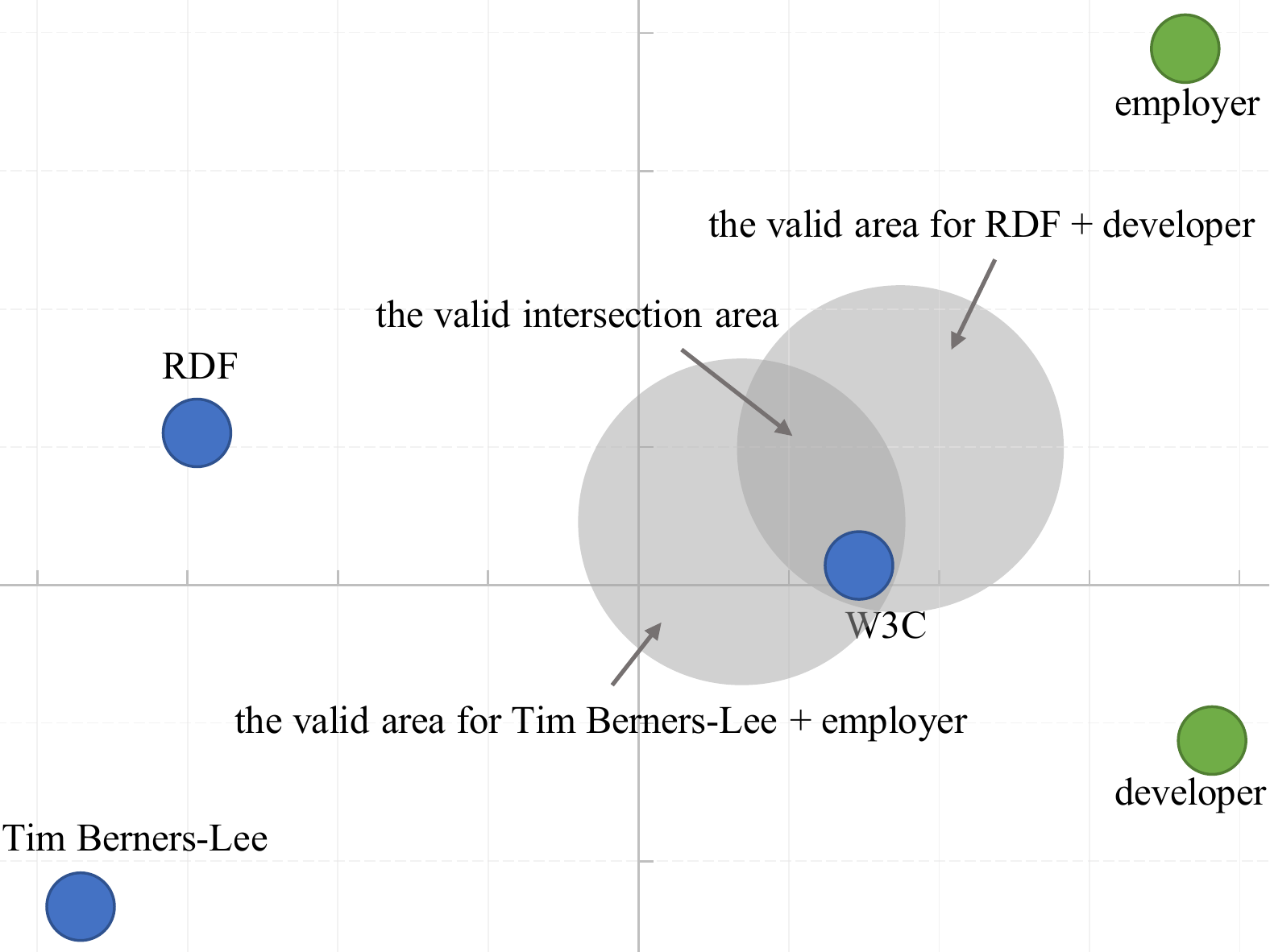}
	\caption{Illustration of margin-based score function. Blue and green nodes denote entity embeddings and relation embeddings, respectively. The center of each grey circle is assigned according to the head entity and relation in the triple. The radii are exactly the margin $\lambda$. }
	\label{fig:margin}
\end{figure}

Another strength of EEA is that it suffers slightly from the heterogeneity of two KGs. The relationships among entities are interpreted by the score function $f_s$ and manifested as distances in the embedding space. Hence, the similarity between a pair of entities can be defined and estimated smoothly. Without loss of generality, we consider the score function of TransE \cite{TransE} as $f_s$. It describes a triple $(e_i, r, e_j)$ by $\rve_i + \rvr \approx \rve_j$, where $e_i, r, e_j$ are the head entity, relation, and tail entity, respectively. The approximation is achieved by defining a margin $\lambda$ to ensure:
\begin{equation}
	|| \rve_i + \rvr - \rve_j || \leq \lambda,
\end{equation}
where $||\cdot||$ denotes the L1 or L2 distance. We illustrate this concept in Figure \ref{fig:margin}, where the two triples \textit{(Tim Berners-Lee, employer, W3C)} and \textit{(RDF, developer, W3C)} have the same tail entity \textit{W3C}. The valid area for \textit{W3C} is decided by two circles. Their centers are $\textit{Tim Berners-Lee} + \textit{employer}$ and $\textit{RDF} + \textit{developer}$, respectively. The radii are exactly the margin $\lambda$. Therefore, the desired embedding of \textit{W3C} should be located in the intersection area.

Many existing EEA methods \cite{MTransE,JAPE,BootEA} have explored how to choose a proper $\lambda$ for the entity alignment task, but we argue this goes beyond a mere parameter-tuning problem. In this paper, we define a paradigm leveraged by the current methods. We show that the embedding discrepancy of an underlying aligned entity pair is bounded by $\epsilon \propto \lambda$, for most EEA methods~\cite{MTransE,JAPE,IPTransE,BootEA,SEA}, or allowing more divergence between two potentially aligned entities~\cite{GCN-Align,RSN,RDGCN,AVR-GCN,AliNet}. Further, we find that this margin-based bound cannot be set as tight as expected, causing minimal constraints on the entities with few neighbors. Take \textit{W3C} in Figure \ref{fig:margin} as an example. The valid area will shrink and finally disappear if this entity has more and more linked neighbors. There is only one way to mitigate this problem -- enlarging the radii, which will allow more divergence for entities with a few neighbors.

We consider additional constraints on entity embeddings to mitigate the above problem, which we name neural-ontology-driven entity alignment (abbr., NeoEA). An ontology~\cite{OWL-ex,OWL2,OWL2-EL} is usually comprised of axioms that define the legitimate relationships among entities and relations. Those axioms make a KG \emph{principled} (i.e., constrained by rules). For example, an ``Object Property Domain'' axiom in OWL 2 EL \cite{OWL2} claims the valid head entities for a specific relation (e.g., the head entities of relation ``birthPlace'' should be in class ``Person''), and it thus determines the head entity distributions of this relation. The neural ontology in this paper is reversely deduced from the entity embedding distributions, which is clearly different from the existing methods like OWL2Vec* \cite{OWL2Vec} that leverages external ontology data to improve KG embeddings. We expect to align the high-level neural ontologies to diminish the discrepancy of entity embedding distributions and ontology knowledge between two KGs.

The main contributions of this paper are threefold:
\begin{itemize}
	\item We define a paradigm for the current EEA methods, and demonstrate that the margin $\lambda$ in their scoring function $f_s$ implicitly bounds the embedding discrepancy in each potential alignment pair. We show that this margin-based bound cannot be as tight as we expect.
	
	\item We propose NeoEA to learn \emph{KG-invariant} and \emph{principled} representations by aligning the neural axioms of two KGs. We prove that minimizing the difference can substantially align their corresponding ontology-level knowledge without assuming the existence of real ontology data.
	
	\item We conducted experiments to verify the effectiveness of NeoEA with several state-of-the-art methods as baselines. We show that NeoEA can consistently and significantly improve their performance.
\end{itemize}

\section{Background}
\label{sec:EEA}

\begin{figure*}[t]
	\centering
	\includegraphics[width=.95\linewidth]{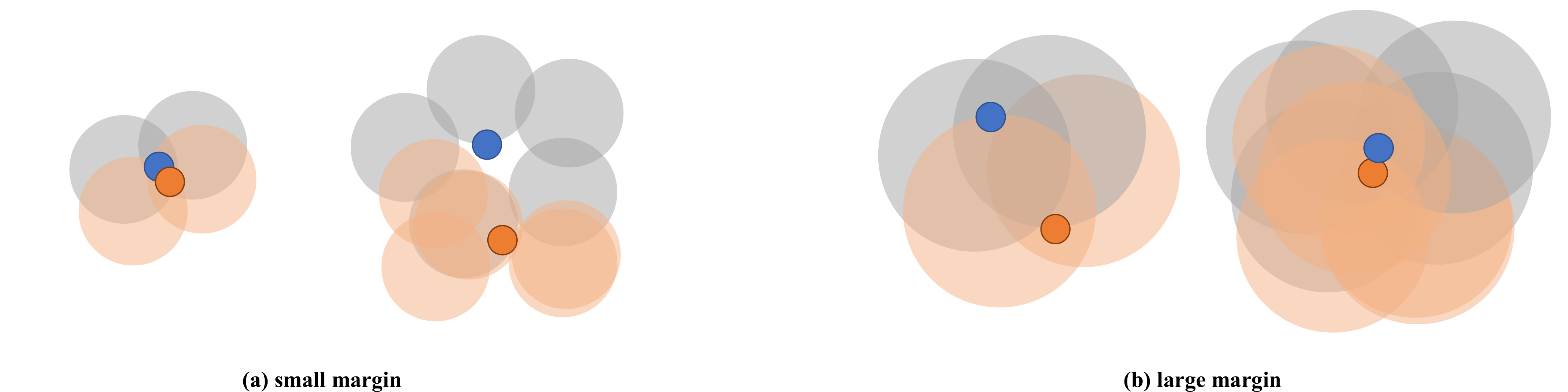}
	\caption{Influence of margin $\lambda$ to different entities in EEA. Blue and orange nodes denote the entities in $\gE_1$ and $\gE_2$, respectively. }
	\label{fig:margin_ea}
\end{figure*}

\subsection{Embedding-based Entity Alignment}

We start by defining a typical KG $\gG = (\gE, \gR, \gT)$, where $\gE$ and $\gR$ are the entity and relation sets respectively, and $\gT \subseteq \gE \times \gR \times \gE$ is the triple set. A triple $(e_1, r, e_2)$ comprises three elements, i.e., the head entity $e_1$,  the relation $r$, and the tail entity $e_2$. We use the boldface $\rve_1$ to denote the embedding of the entity $e_1$.

The common paradigm employed by most existing EEA methods~\cite{MTransE,JAPE,IPTransE,BootEA,SEA} is then defined as:

\begin{definition}[Embedding-based Entity Alignment]
	The input of EEA is two KGs $\gG_1 = (\gE_1, \gR_1, \gT_1)$, $\gG_2 = (\gE_2, \gR_2, \gT_2)$, and a small subset of aligned entity pairs $\mathcal{S} \subset \mathcal{E}_1 \times \mathcal{E}_2$ as seeds to connect $\gG_1$ with $\gG_2$. An EEA model consists of two neural functions: an alignment function $f_a$, which is used to regularize the embeddings of pairwise entities in $\gS$; and a scoring function $f_s$, which scores the embeddings based on the joint triple set $\gT_1\cup \gT_2$. EEA estimates the alignment of an arbitrary entity pair $(e_i^1, e_j^2)$ by their geometric distance $d(\rve_i^1,  \rve_j^2)$.
\end{definition}

The existing studies have explored a diversity of $f_a$. The pioneering work MTransE~\cite{MTransE} was proposed to learn a mapping matrix to cast an entity embedding $\rve_i^1$ to the vector space of $\gG_2$. SEA~\cite{SEA} and OTEA~\cite{OTEA} extended this approach with adversarial training to learn the projection matrix. Especially, OTEA is highly related to our approach as it is also based on optimal transport (OT)~\cite{wGAN}. The differences are: (1) NeoEA provides a general way to align entity embedding distributions, while OTEA regularizes the projection matrix via optimal transport. (2) NeoEA exploits the underlying ontology information to facilitate entity alignment.

Recently, a simpler yet more efficient method was widely-used, which directly maps a pair of aligned entities $(e_i^1, e_i^2) \in \gS$ to one embedding vector $\rve_i$~\cite{JAPE,IPTransE,RSN}. Meanwhile, some methods \cite{GCN-Align,SEA,RDGCN} started to leverage a softer way to incorporate seed information, in which the distance between entities in a positive pair (i.e., known alignment in $\mathcal{S}$) is minimized, while that referred to a negative pair is enlarged. As the most efficient choice, we consider $f_a$ as Euclidean distance between two embeddings \cite{AliNet,RSN,GCN-Align,SEA,RDGCN}. The corresponding alignment loss can be written as follows:
\begin{equation}
\begin{split}
\label{eq:reg}
\mathcal{L}_a = &\sum_{(e_i^1,e_i^2) \in \mathcal{S}}||\rve_i^1 - \rve_i^2|| + 
\sum_{(e_{i'}^1,e_{j'}^2) \in \mathcal{S}^{-}}\mathit{ReLU}(\alpha - || \rve_{i'}^1 - \rve_{j'}^2||),
\end{split}
\end{equation}
where $\mathcal{S}^{-}$ denotes the set of negative pairs. $\alpha$ is the minimal margin allowed between entities in each negative entity pair.

On the other hand, the scoring function $f_s$ also has various design choices \cite{JAPE,RSN,GCN-Align}. Most methods \cite{MTransE,JAPE,SEA,OTEA,BootEA} choose TransE as their scoring function, i.e., $f_s(\tau) = f_s((e_i, r, e_j)) = ||\rve_i + \rvr - \rve_j||, \tau=(e_i, r, e_j)\in \gT_1 \cup \gT_2$. The corresponding loss is:
\begin{equation}
\begin{split}
\label{eq:score}
\mathcal{L}_s = & \sum_{\tau \in \gT_1 \cup \gT_2} \mathit{ReLU}(f_s(\tau)-\lambda) + 
\sum_{\tau' \in \gT_1^- \cup \gT_2^-} \mathit{ReLU}(\lambda - f_s(\tau')),
\end{split}
\end{equation}
where $\gT_1^-$ and $\gT_2^-$ are negative triple sets. $\mathcal{L}_s$ is a margin-based loss in which the distance $d(\rve_i + \rvr, \rve_j)$ in a positive triple should at least be smaller than $\lambda \geq 0$, while larger than $\lambda$ for negative ones. 

\subsection{Understanding and Rethinking EEA}

We illustrate how an EEA method works with Figure \ref{fig:margin_ea}. When the KG embedding model has a small margin, as shown in the left of Figure \ref{fig:margin_ea}a, the entities with few neighbors can be constrained tightly. With aligned entity pairs serving as anchors, the circles are very close to each other. Therefore, two entities stay closely in the overlapped intersection areas. By contrast, there is no valid area for the entities with rich neighbors. The entities in the right of Figure \ref{fig:margin_ea}a are not ``fully expressed'' \cite{SimplE,ComplEx}.

If we enlarge the margin, as shown in Figure \ref{fig:margin_ea}b, the embeddings for entities with rich neighbors can be correctly assigned. However, the intersection areas for entities with few neighbors are too loose to bound the underlying aligned entities. The two embeddings are not as similar as we expect in the vector space. We summarize the above observations as:

\begin{proposition}[Discrepancy Bound]
	\label{proposition:bound}
	The embedding difference of two potentially aligned entities $(e_x^1, e_y^2)$ is bound by $\epsilon$, which is proportional to the hyper-parameter $\lambda$:
	\begin{align}
		\exists \epsilon \propto \lambda, ||\rve_x^1 - \rve_y^2|| \leq \epsilon.
	\end{align}
\end{proposition}
\begin{proof}
	We start with the case that each entity in  $(e_x^1, e_y^2)$ has only one neighbor, connected by the same relation $r^1 = r^2$. We assume that their neighbors are actually a pair of aligned entities $(e_i^1, e_i^2) \in \gS$. With a well-trained and almost optimal EEA model, we have $\rve_i^1 = \rve_i^2$ (as $\mathcal{L}_a$ is minimized) and $\rvr^1=\rvr^2$ (denoted by $\rvr$ for simplicity). According to Equation~(\ref{eq:score}), we have:
	\begin{align}
		||f_s(\rve_x^1, \rvr, \rve_i^1)|| \approx ||f_s(\rve_y^2, \rvr, \rve_i^2)|| \leq \lambda.
	\end{align}
	Without loss of generality, we consider the scoring function of TransE as $f_s$, and then derive:
	\begin{align}
		\label{eq:score_func}
		||\rve_x^1 + \rvr - \rve_i^1|| \leq \lambda,\quad||\rve_y^2 + \rvr - \rve_i^2|| \leq \lambda.
	\end{align}
	For simplicity, we use a constant $C$ to denote $\rvr - \rve_i^1$ and $\rvr - \rve_i^2$, such that Equation~(\ref{eq:score_func}) will be rewritten as
	\begin{align}
		\label{eq:rewrite_score_func}
		||\rve_x^1 + C|| \leq \lambda,\quad||\rve_y^2 + C|| \leq \lambda.
	\end{align}
	Then, we get
	\begin{align}
		2 \lambda \geq & ||\rve_x^1 + C|| + ||\rve_y^2 + C|| \nonumber\\
		\geq&||(\rve_x^1 + C) - (\rve_y^2 + C)|| \nonumber\\ 
		= & ||\rve_x^1-\rve_y^2||.
	\end{align}
	Now, we consider a more complicated case, where the neighbors of $e_x^1$ and $e_y^2$ are not in the known alignment set. We denote $\rvr - \rve_i^1$ and $\rvr - \rve_i^2$ by $C^1_i$ and $C^2_i$, respectively. If the neighbors of the neighbors $e_i^1, e_i^2$ are a pair of known alignment, we will have $||C^1_i - C^2_i|| = ||\rve_i^2- \rve_i^1|| \leq 2\lambda$, otherwise we can recursively navigate more neighbors. Therefore, we have:
	\begin{align}
		2 \lambda \geq & ||\rve_x^1-\rve_y^2 + (C^1_i - C^2_i)|| \nonumber\\
		\geq & ||\rve_x^1-\rve_y^2|| - ||C^1_i - C^2_i||,
	\end{align}
	which results to an looser bound:
	\begin{align}
		||\rve_x^1-\rve_y^2|| \leq 2\lambda + ||C^1_i - C^2_i||.
	\end{align}
	For the case that the entities $(e_x^1, e_y^2)$ have more than one neighbors, the bound will be further tighten as the embeddings are constrained by multiple triples.
\end{proof}

Proposition~\ref{proposition:bound} suggests that decreasing the value of $\lambda$ will decrease the embedding discrepancy in the underlying aligned entity pairs. However, previous studies \cite{ComplEx,SimplE} have proved that $\lambda$ cannot be set as small as we want. This is because TransE with a small margin is not sufficient to fully capture the semantics contained in triples. Some empirical statistics~\cite{BootEA} also illustrate such results. Enlarging the margin $\lambda$, on the other hand, will bring significant variance between $\rve_x^1$ and $\rve_y^2$, especially for those entities with few neighbors. For the models that do not belong to the TransE family, e.g., neural-based like ConvE~\cite{ConvE}, or composition-based like ComplEx~\cite{ComplEx}, as proved in ~\cite{BilinearAAAI18,SimplE}, they are more expressive than TransE. In this case, entities with sufficient neighbors can be correctly modeled, while entities with only a few neighbors are also less constrained. Therefore, those models allow more diversity between $\rve_x^1$ and $\rve_y^2$. We believe this is why they performed badly in the EA task \cite{RSN,OpenEA}.

\begin{figure*}[htb]
	
	\centering
	\hfill
	\begin{subfigure}[h]{0.195\textwidth}
		\includegraphics[width=\textwidth]{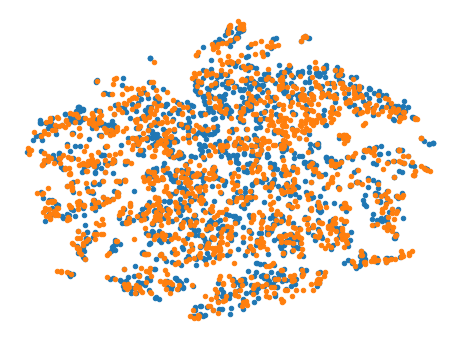}
		\caption{}\label{fig:dista}
	\end{subfigure}
	\hfill
	\begin{subfigure}[h]{0.195\textwidth}
		\includegraphics[width=\textwidth]{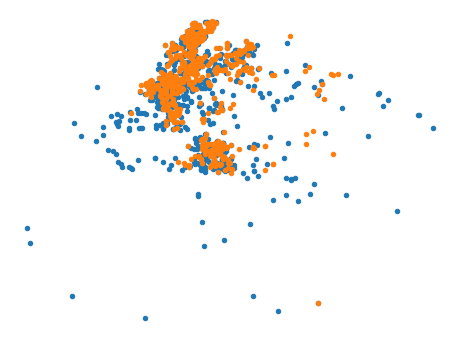}
		\caption{}\label{fig:distb}
	\end{subfigure}
	\hfill
	\begin{subfigure}[h]{0.195\textwidth}
		\includegraphics[width=\textwidth]{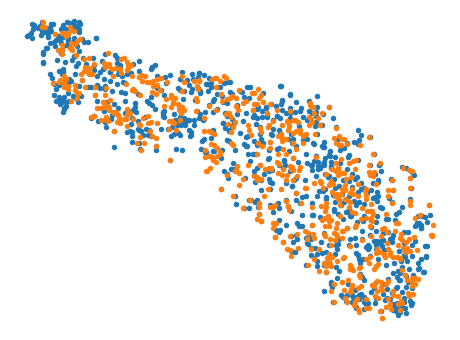}
		\caption{}\label{fig:distc}
	\end{subfigure}
	\hfill
	\begin{subfigure}[h]{0.195\textwidth}
		\includegraphics[width=\textwidth]{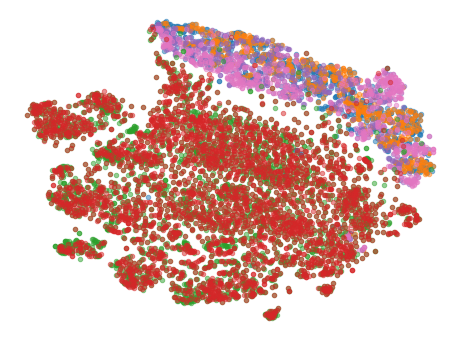}
		\caption{}\label{fig:distd}
	\end{subfigure}
	\hfill
	\begin{subfigure}[h]{0.195\textwidth}
		\includegraphics[width=\textwidth]{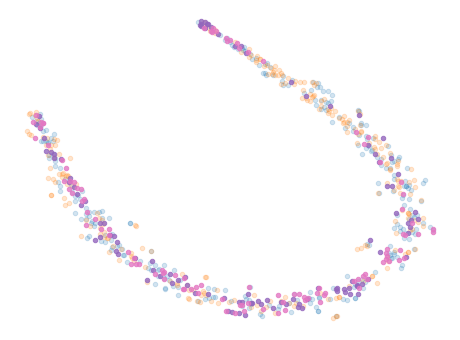}
		\caption{}\label{fig:diste}
	\end{subfigure}
	
	\caption{Example of different embedding distributions. (a) Entity embedding distributions of two KGs, i.e., $\sA_{E}$. $Blue$ points denote entities in $\gG_1$, and $orange$ ones are entities in $\gG_2$. The two embedding sets are nearly uniformly distributed and almost aligned (based on the EEA model RDGCN~\protect\cite{RDGCN}, the same below). (b) The head entity embedding distributions of relation ``genre''. The two distributions are only aligned partially. (c) Head entity embedding distributions conditioned on ``genre'', i.e., $\sA_{E_h|r_i}$. Two conditioned distributions are aligned as expected. (d) The head entity distributions conditioned on three different relations: ``genre'' (colors: $<blue,orange>$), ``writer'' (colors: $<purple,pink>$), ``brithPlace'' (colors: $<green,red>$). The distributions corresponding to the first two relations are overlapped, while a clear decision boundary between them and the last one is observed. (e) Triple embedding distributions conditioned on relations ``artist'' (colors: $<blue,orange>$) and ``musicalArtist'' (colors: $<purple,pink>$), respectively. The entity embeddings referred to sub-relation ``musicalArtist'' are covered by those corresponding to ``artist''.}
	\label{fig:dist}
\end{figure*}

In short, most existing works adopt the above strategy to learn cross-KG embeddings for EA, which makes them stuck in balancing between the bound and the expressiveness. On the one hand, they want the KG embedding model can fully model all given triples. On the other hand, they also want the discrepancy between potentially aligned entities to be restrained more tightly. In this paper, we explore a new direction to align the conditioned embedding distributions of two KGs to ensure the embeddings \emph{principled}.

\section{Neural Ontology}
\label{sec:Neo}

\subsection{Aligning Embedding Distributions with Adversarial Learning}
In real-world KGs, entities conform with the axioms in ontologies~\cite{OWL2}. Similarly, we call the entity embedding distributions ``neural axioms'', as aligning them also allows us to regularize the entity embeddings at a high level.

We start from an introduction to the entity embedding distribution. It is well-known that entity embeddings can implicitly capture some ontology-level information~\cite{TransE, DistMult}. For example, entities that belong to the same class are usually spatially close to each other in the vector space. In the other way around, a cluster of entity embeddings in the vector space may also indicate the existence of a class. Our goal is to exploit such ontology-level knowledge from the embedding distributions. Therefore, we define the basic neural axiom as the distribution of entity embeddings: 
	\begin{equation}
	\begin{aligned}
    \sA_{E} =&  p_e (\rve),
	\end{aligned}
	\end{equation}
where $p_e$ is the entity probability distribution over the sample set $E$ (in this case it equals to $\gE$). Aligning the basic neural axioms $\sA_{E}^1$ and $\sA_{E}^2$ of two KGs is trivial. We take the advantages of existing domain adaptation (DA) methods~\cite{DANN,WDGRL,DATheory,JDOT} that also aims to align the feature distributions of two datasets for knowledge transferring. Specifically, we consider the method based on adversarial learning \cite{GAN,wGAN}, where a discriminator is employed to distinguish entity embeddings of $\gE_1$ from those of $\gE_2$ (or vice versa). The embeddings, by contrast, try to confuse the discriminator. Therefore, the same semantics in two KGs shall be encoded in the same way into the embeddings to fool the discriminator. The corresponding empirical Wasserstein distance based loss~\cite{wGAN,WDGRL} is:
\begin{equation}
\label{eq:wd}
\mathcal{L}_{\sA_{E}} = \E_{\sA_{E}^1} [f_w(\rve)] - \E_{\sA_{E}^2} [f_w(\rve)],
\end{equation}
where $f_w$ is the learnable domain critic that maps the embedding vector to a scalar value. As suggested in \cite{wGAN}, the empirical Wasserstein distance can be approximated by maximizing $\mathcal{L}_{\sA_{E}}$, if the parameterized family of $f_w$ are all 1-Lipschitz.

Although the above method provide a general solution for many alignment tasks, it is not completely appropriate to the EEA problem. The most important reason is that entity embeddings are initialized randomly and tend to uniformly distributed in the vector space, which we can observe from Figure~\ref{fig:dista}. 

Recall that the alignment loss $\mathcal{L}_a$ consists of two terms. The first is
$\sum_{(e_i^1,e_i^2) \in \mathcal{S}}||\rve_i^1 - \rve_i^2||$,
which aims to minimize the difference of embeddings for each positive pair. The cardinality of $\mathcal{S}$ is usually small in the weakly supervised setting. However, a large size of negative samples are used for contrastive learning, which means that $||\mathcal{S}|| \ll ||\mathcal{S}^-||$. The model actually put more effort into the second term
$\sum_{(e_{i'}^1,e_{j'}^2) \in \mathcal{S}^{-}}\mathit{ReLU}(\alpha - || \rve_{i'}^1 - \rve_{j'}^2||)$,
of which the main target is to randomly push the embeddings of non-aligned entities away from each other. On the other hand, $\mathcal{L}_s$ is also a contrastive loss and has a similar effect on maximizing the pairwise distance between a positive entity and its corresponding sampled negative ones. Therefore, we may conclude that: 
\begin{proposition}[Uniformity]
	\label{prop:uniformity}
	The entity embeddings of two KGs tend to be uniformly distributed in the vector space as an EEA model is optimized.
\end{proposition}
\noindent This characteristic has also been studied in~\cite{Uniformity} in other representation learning problems. It is actually a good property revealing that the information of entities are efficiently encoded to maximize the entropy. However, for the given EEA problem, the entity embedding distributions of two KGs will be similar to each other, especially when the seed alignment pairs exist. Thus, only aligning the basic neural axioms is less helpful to facilitate EEA.

\begin{figure*}[t]
	\centering
	\includegraphics[width=.98\textwidth]{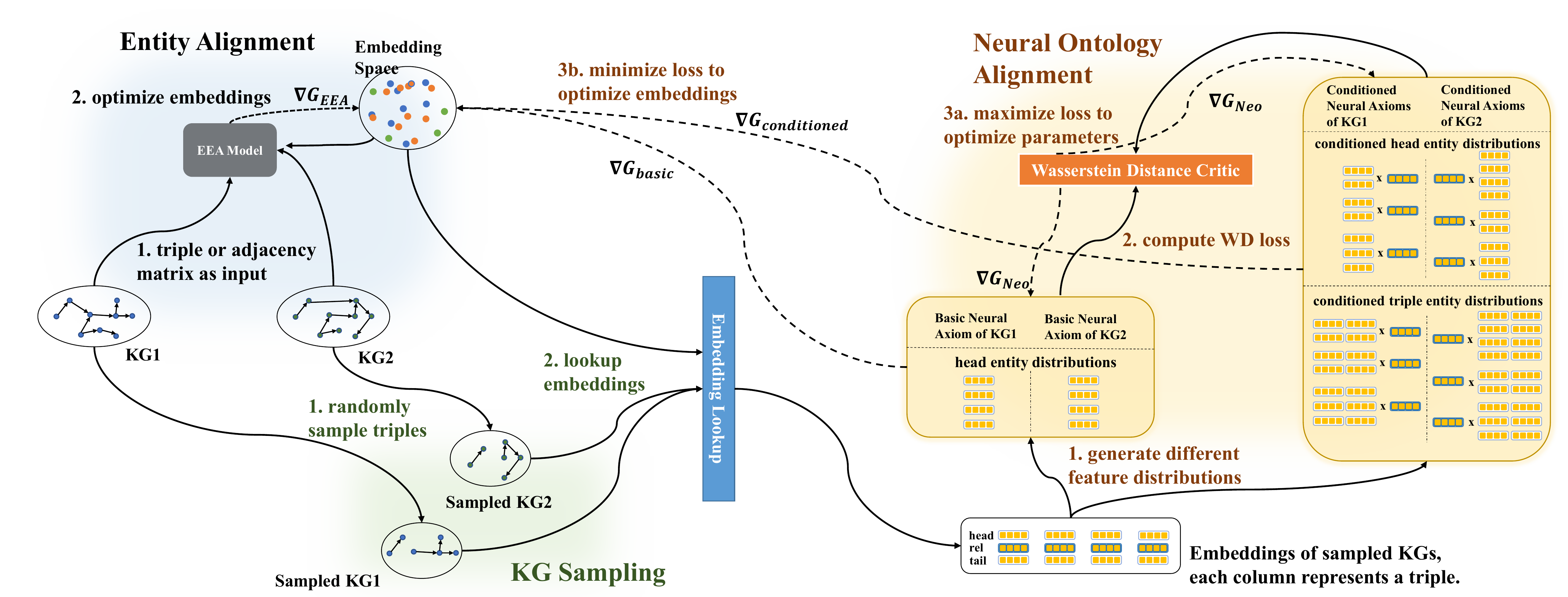}
	\caption{Architecture of NeoEA. Solid lines denote forward propagation, while dotted lines represent backward propagation. The architecture consists of three decoupled modules: (1) Entity alignment module, in which the EEA model is unaware of the existence of other modules. (2) KG sampling module, in which we sample sub-KGs to replace the whole KGs for efficiency. (3) Neural ontology alignment module, in which the discrepancy between each pair of neural axioms is estimated by Wasserstein-distance critic and optimized by gradient ascent/descent.}
	\label{fig:arch}
\end{figure*}

\subsection{Conditional Neural Axiom}
\label{sec:conNA}

We can estimate the conditional distributions rather than the raw distributions to avoid the problem brought from the uniformity property. Specifically, we name conditional neural axioms to describe the entity (or triple) embedding distributions under specific semantics:
	\begin{align}
	\sA_{E_h|r_i} &= p_{e_h|r_i} (\rve_h|\rvr_i),
	\end{align}
where $\sA_{E_h|r_i}$ denotes the head entity embedding distribution conditioned on the relation embedding $\rvr_i$. $p_{e_h|r_i}$ is the conditional probability distribution of the head entities given $r_i$. The corresponding sample set is defined as:
	\begin{align}
	E_h|r_i &= \{e \,|\, \exists e^\prime, (e, r_i, e^\prime) \in \gT \}.
	\end{align}
Following the similar rule, we can define the conditioned triple distribution $\sA_{E_{h,t}|r_i}$ with the sample set 
\begin{align}
E_{h,t}|r_i = \{(e_h, e_t)\,|\,(e_h, r_i, e_t) \in \gT \}.
\end{align}
Numerous methods have been proposed to process the neural conditioning operation, ranging from addition and concatenation~\cite{cGAN,TransH, ConvE}, to matrix multiplication~\cite{TransR,TransD,STransE}. Rather than developing new methods, we value more on its common merit: projecting the entities to a relation-specific subspace~\cite{TransH,TransR,STransE}. Hence, the corresponding embedding distributions conditioned on different relations become discriminative, compared to uniformly distributed in the original embedding space.

Furthermore, conditional neural axioms capture high-level ontology knowledge:
\begin{proposition}[Expressiveness]
    \label{proposition:expressiveness}
	Aligning the conditional neural axioms minimizes the embedding discrepancy of two KGs at the ontology level.
\end{proposition}
\begin{proof}
	See Appendix~\ref{app:neoea_proof} for details. 
\end{proof}

	We take $\sA_{E_h|r_i}$ as an example that summarizes the empirical ``Object Property Domain'' axiom of $r_i$ in OWL 2 EL~\cite{OWL2-EL}. Supposed there exists such an axiom stating that the head entities of $r_i$ should belong to some specific class $c$ (e.g., only head entities belonging to the class "Person" have the relation "birthPlace"). We further suppose that there exists a classifier $f_c (\rve) \in [0, 1]$, such that $f_c(\rve_j) = 1$ if head entity $e_j$ belongs to class $c$, and $0$ otherwise. Then, with the knowledge of the given axiom, one may derive the following rule:
	\begin{equation}
	\forall e \in \{e\,|\,\exists e^\prime, (e, r_i, e^\prime) \in \gT_1\cup \gT_2 \},\ f_c(\rve) = 1,
	\end{equation}
	which is equivalent to:
	\begin{equation}
	\begin{aligned}
	\E_{\sA^1_{E_h|r_i}} [f_c(\rve)] = \E_{\sA^2_{E_h|r_i}} [f_c(\rve)] = 1,
	\end{aligned}
	\end{equation}
	both of which means that all head entities of $r_i$ in either KG should be correctly classified to $c$. Then, we have:
	\begin{equation}
	\E_{\sA_{E_h|r_i}^1} [f_c(\rve)] - \E_{\sA_{E_h|r_i}^2} [f_c(\rve)] = 0.
	\end{equation}
	In fact, we do not have such knowledge about $r_i$ and class $c$. Instead, we can leverage a neural function $f_{c^\prime}(\rve|\rvr_i)$ to estimate $f_c$ empirically. In this way, $\sA_{E_h|r_i}^1$ and $\sA_{E_h|r_i}^2$ are supposed to be aligned to minimize the loss corresponding to the above rule. Therefore, we deduce this problem back to a similar form to Equation~\ref{eq:wd}, i.e.,
	
	\begin{equation}
	\label{eq:cwd}
	\mathcal{L}_{\sA_{E_h|r_i}} = \E_{\sA_{E_h|r_i}^1} [f_{c^\prime}(\rve\,|\,\rvr_i)] - \E_{\sA_{E_h\,|\,r_i}^2} [f_{c^\prime}(\rve|\rvr_i)],
	\end{equation}
	
	 \noindent which suggests that aligning the above conditional neural axioms can minimize the discrepancy of potential ``Object Property Domain'' axioms between two KGs.

\begin{example}[OWL2 axiom: ObjectPropertyDomain]
As shown in Figure~\ref{fig:distb} and Figure~\ref{fig:distc}, we assume that the head entities of relation ``genre'' are under the class ``Work of Art'' (although it does not exist in the dataset). It is clear that the head entity embedding distributions are only partially aligned in Figure~\ref{fig:distb}, while those in Figure~\ref{fig:distc} are matched well. 

In Figure~\ref{fig:distd}, we illustrate a more complicated example. The head entities of relations ``genre'' and ``writer'' mainly belong to ``Work of Art'', which show overlapped distributions (blue-orange, pink-purple) in the figure. By contrast, there exists a clear decision boundary between them and the distributions conditioned on relation ``birthPlace'' (red-green), as the head entities of relation ``birthPlace'' are under the class ``Person''.
\end{example}

\begin{example}[OWL2 axiom: SubObjectPropertyOf]
We consider two relations, ``musicalArtist'' and ``artist'' as an example, where the former is the latter's sub-relation. In Figure~\ref{fig:diste}, the triple distributions conditioned on ``musicalArtist'' (colors: pink-purple) are covered by those conditioned on ``artist'' (colors: orange-blue).
\end{example}

\begin{algorithm}[t]
	\caption{NeoEA}
	\label{alg:NeoEA}
	\begin{algorithmic}[1]
		\STATE {\bfseries Input:} two KGs $\gG_1$, $\gG_2$, the seed alignment set $\gS$, the EEA model $\gM(f_s, f_a)$, number of steps for NeoEA $n$;
		\STATE Initialize all variables;
		\REPEAT
		\FOR{$i:=1$ {\bfseries to} $n$}
		\STATE Sample sub-KGs from respective KGs $\gG_1$, $\gG_2$;
		\STATE Compute the Wasserstein-distance-based loss $\mathcal{L}_{w}$ for each pair of neural axioms;
		\STATE Optimize the Wasserstein distance critic $f_w$ by maximizing $\mathcal{L}_{w}$.
		\ENDFOR
		
		\STATE Sample sub-KGs from respective KGs $\gG_1$, $\gG_2$;
		\STATE Compute Wasserstein-distance-based loss $\mathcal{L}_{w}$ for each pair of neural axioms;
		\STATE Compute the losses $\mathcal{L}_{a}$, $\mathcal{L}_{s}$ of the EEA model $\gM$;
		\STATE Optimize the EEA model and embeddings by minimizing $\mathcal{L}_{a}$, $\mathcal{L}_{s}$, $\mathcal{L}_{w}$;
		
		\UNTIL{the alignment loss on the validation set converges.}
	\end{algorithmic}
\end{algorithm}

\subsection{Implementation}
\label{app:arch}
We illustrate the overall structure and training procedure in Figure~\ref{fig:arch} and Algorithm~\ref{alg:NeoEA}, respectively. The framework can be divided into three modules:

\textbf{Entity Alignment.} This module aims at encoding the semantics of KGs into embeddings. Almost all existing EEA models can be used here, no matter what the input data look like (e.g., triples or adjacency matrices).

\textbf{KG Sampling.} For each KG, we sample a sub-KG to estimate the data distributions of neural axioms. It is more efficient than separately sampling candidates for each axiom, especially when KGs get big.

\textbf{Neural Ontology Alignment.} As aforementioned, for each pair of embedding distributions, we align them by minimizing the empirical Wasserstein distance.

For efficiency, we share the parameters of Wasserstein distance critic only in each \emph{type} of neural axiom, which reduces the number of model parameters and avoids the situation that some relations only have a small number of triples. This also allows us to perform fast mini-batch training by aligning the axioms of the same type in one operation. Given the sample KGs $\gG'_1 = (\gE'_1, \gR'_1, \gT'_1)$, $\gG'_2 = (\gE'_2, \gR'_2, \gT'_2)$, the corresponding batch loss is:
\begin{align}
	\label{eq:batch_loss}
	\mathcal{L}_{\mathit{sep}} =& \mathcal{L}_{\sA_{E'}} \nonumber\\
	+& ( \sum_{r' \in \gR'_1}{\E_{\sA_{E'_h|r'}}[f_{h|r}(\rve|\rvr')]} \nonumber\\
	&- \sum_{r' \in \gR'_2}{\E_{\sA_{E'_h|r'}}[f_{h|r}(\rve|\rvr')]} ) \nonumber\\
	+& ( \sum_{r' \in \gR'_1}{\E_{\sA_{E'_{h,t}|r'}}[f_{h,t|r}(\rve_h, \rve_t|\rvr')]} \nonumber\\
	&- \sum_{r' \in \gR'_2}{\E_{\sA_{E'_{h,t}|r'}}[f_{h,t|r}(\rve_h, \rve_t|\rvr')]} ) 
\end{align}
where $\mathcal{L}_{\sA_{E'}}$ is the basic axiom loss under the sampled KGs. $f_{h|r}$ and $f_{h,t|r}$ are the critic functions of two types of neural axioms, respectively. The loss $\mathcal{L}_{\mathit{batch}}$ will approximate to that in pairwise calculation when batch-size is considerably greater than the number of relations. We take the second term in Equation~(\ref{eq:batch_loss}) as an example. For pairwise estimation, the loss should be:
\begin{align}
	\sum&_{(r_1, r_2)\in \gS_r}{(\E_{\sA_{E'_h|r_1}}[f_{h|r}(\rve|\rvr_1)] - \E_{\sA_{E'_h|r_2}}[f_{h|r}(\rve|\rvr_2)])}\nonumber\\
	=&\sum_{(r_1, r_2)\in \gS_r}{\E_{\sA_{E'_h|r_1}}[f_{h|r}(\rve|\rvr_1)]}\nonumber\\
	-&\sum_{(r_1, r_2)\in \gS_r}{\E_{\sA_{E'_h|r_2}}[f_{h|r}(\rve|\rvr_2)]},
\end{align}
where $\gS_r \subset \gR_1 \times \gR_2$ denotes the set of all aligned relation pairs. The above equation suggests that the pairwise loss is based on the respective relation sets of two KGs, not constrained by each pair of aligned relations. Generally, the number of relations is much smaller than the number of sampled triples in one batch, which means that $\gR'_1, \gR'_2$ in Equation~(\ref{eq:batch_loss}) can cover a large proportion of elements in the full relation sets $\gR_1$, $\gR_2$. Therefore, we used $\mathcal{L}_{\mathit{batch}}$ to approximate the pairwise loss in the implementation.

\section{Experiments}
\label{sec:expr}

\begin{table*}[!t]
	\centering
	\caption{Results on V1 datasets (5-fold cross-validation).}
	\resizebox{.82\textwidth}{!}{
		\begin{tabular}{lcccccccccccc}
			\toprule
			\multirow{2}{*}{Models} & \multicolumn{3}{c}{EN-FR} & \multicolumn{3}{c}{EN-DE} & \multicolumn{3}{c}{D-W} & \multicolumn{3}{c}{D-Y}\\
			\cmidrule(lr){2-4} \cmidrule(lr){5-7} \cmidrule(lr){8-10} \cmidrule(lr){11-13}
			& H@1 & H@5 & MRR & H@1 & H@5 & MRR & H@1 & H@5 & MRR & H@1 & H@5 & MRR \\ \midrule
			RSN \cite{RSN}& .393 & .595 & .487 & .587 & .752 & .662 & .441 & .615 & .521 & .514 & .655 & .580 \\
			RSN\,+\,NeoEA & \textbf{.399} & \textbf{.597} & \textbf{.490} & \textbf{.600} & \textbf{.759} & \textbf{.673} & \textbf{.450} & \textbf{.624} & \textbf{.530} & \textbf{.522} & \textbf{.663} & \textbf{.588}\\
			\midrule
			SEA \cite{SEA} & .280 & .530 & .397 & .530 & .718 & .617 & .360 & .572 & .458 & .500 & .706 & .591 \\
			SEA\,+\,NeoEA & \textbf{.320} & \textbf{.584} & \textbf{.443} & \textbf{.586} & \textbf{.766} & \textbf{.668} & \textbf{.389} & \textbf{.608} & \textbf{.490} & \textbf{.549} & \textbf{.752} & \textbf{.638}\\
			\midrule
			BootEA \cite{BootEA} & .507 & .718 & .603 & .675 & \textbf{.820} & \textbf{.740} & .572 & .744 & .649 & .739 & .849 & .788 \\
			BootEA\,+\,NeoEA & \textbf{.521} & \textbf{.733} & \textbf{.617 }& \textbf{.676} & \textbf{.820} & \textbf{.740} & \textbf{.579} & \textbf{.753} & \textbf{.658} & \textbf{.756} & \textbf{.859} & \textbf{.797}\\
			\midrule
			RDGCN \cite{RDGCN}& .755 & .854 & .800 &.830 & .895 & .859 & .515 & .669 & .584 & .931 & .969 & .949 \\
			RDGCN\,+\,NeoEA & \textbf{.775} & \textbf{.868} & \textbf{.817} & \textbf{.846} & \textbf{.908} & \textbf{.874} & \textbf{.527} & \textbf{.671} & \textbf{.592} & \textbf{.941} & \textbf{.972} & \textbf{.955}\\
			\bottomrule
	\end{tabular}}
	\label{tab:main_results_v1}
\end{table*}

\begin{figure*}[t]
	\centering
	\includegraphics[width=.95\linewidth]{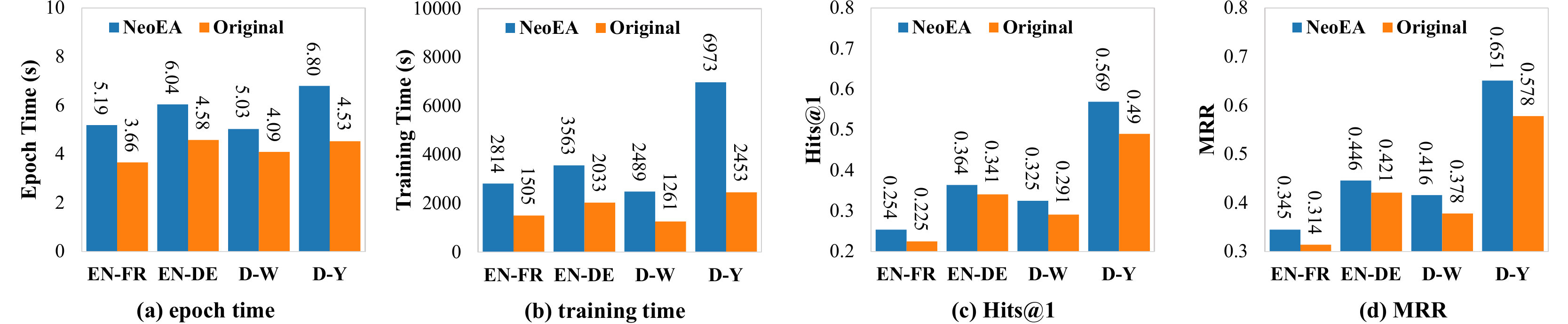}
	\caption{Results on OpenEA 100K datasets.}
	\label{fig:100k}
\end{figure*}

In this section, we empirically verify the effectiveness of NeoEA by a series of experiments with state-of-the-art methods as baselines. The source code was uploaded \footnote{https://anonymous.4open.science/r/NeoEA-9DB5}.
\subsection{Settings}
We selected four best-performing and representative models as our baselines:
\begin{itemize}
    \item RSN~\cite{RSN},  an RNN-based EEA model with only structure data as input.
    \item SEA~\cite{SEA},  a TransE-based model with both structure and attribute data as input.
    \item BootEA~\cite{BootEA},  a TransE-based EEA model with only structure data as input.
    \item RDGCN~\cite{RDGCN}, a GCN-based model with both structure and attribute data as input.
\end{itemize}
The whole framework is based on OpenEA~\cite{OpenEA}. We modified only the initialization of the original project. In this sense, the EEA models were unaware of the existence of neural ontologies. Furthermore, we kept the optimal hyper-parameter settings in OpenEA to ensure a fair comparison.

\sloppy The data distributions of some previous benchmarks such as JAPE~\cite{JAPE} and BootEA~\cite{BootEA} are different from those of real-world KGs, which means that conducting experiments on those benchmarks cannot reflect the realistic performance of an EEA model~\cite{RSN,OpenEA}. Therefore, we consider the latest benchmark provided by OpenEA~\cite{OpenEA}, which consists of four sub-datasets with two density settings. Specifically, ``D-W'', ``D-Y'' denote ``DBpedia~\cite{DBpedia}-WikiData~\cite{Wikidata}'', ``DBpedia-YAGO~\cite{Yago}'', respectively. ``EN-DE'' and ``EN-FR' denote two cross-lingual datasets, both of which are sampled from DBpedia. Each sampled KG has around 15,000 entities. The entity degree distributions in ``V1'' datasets are similar to those in the original KGs, while the average degree in ``V2'' datasets are doubled. For detailed statistics of the datasets, please refer to \cite{OpenEA}.

\subsection{Empirical Comparisons}

\begin{table*}[!t]
	\centering
	\caption{Results of ablation study based on the best-performing model RDGCN, on V1 datasets.}
	\resizebox{.82\textwidth}{!}{
		\begin{tabular}{lcccccccccccc}
			\toprule
			\multirow{2}{*}{Models} & \multicolumn{3}{c}{EN-FR} & \multicolumn{3}{c}{EN-DE} & \multicolumn{3}{c}{D-W} & \multicolumn{3}{c}{D-Y}\\
			\cmidrule(lr){2-4} \cmidrule(lr){5-7} \cmidrule(lr){8-10} \cmidrule(lr){11-13}
			& H@1 & H@5 & MRR & H@1 & H@5 & MRR & H@1 & H@5 & MRR & H@1 & H@5 & MRR \\ \midrule
			Full & \textbf{.775} & \textbf{.868} &\textbf{.817} &\textbf{ .846} &\textbf{.908} &\textbf{.874} & \textbf{.527} &\textbf{.671} &\textbf{.592} & \textbf{.941} &\textbf{.972} &\textbf{.955} \\
			Partial & .771 & .863 & .813 & .840 &.900 &.871 & .523 &.669 &.590 & .936 &.971 &.952\\
			Basic & .755 & .853 & .799 & .827 &.895 &.858 & .512 &.656 &.578 & .931 &.969 &.948 \\
			Original & .755 & .854 & .800 & .830 & .895 &.859 & .515 & .669 & .584 &.931 &.969 &.949 \\
			\bottomrule
	\end{tabular}}
	\label{tab:ablation_study}
\end{table*}

The main results on V1 datasets are shown in Table~\ref{tab:main_results_v1}. Although the performance of four baseline models varied from different datasets, all of them gained improvement with NeoEA. This demonstrates that aligning the neural ontology is beneficial for all four kinds of EEA models. 

On the other hand, we find that the performance improvement on SEA and RDGCN was more significant than that on BootEA and RSN, as the latter two are not typical EEA models. BootEA has a sophisticated bootstrapping procedure, which may be challenging to be injected with NeoEA. RSN tries to capture long-term dependencies. The complicated objective may conflict with NeoEA more or less. Even though we still observe relatively significant improvement on some datasets (e.g., EN-DE and D-Y). Therefore, we believe their performance can be refined through a joint hyper-parameter turning with NeoEA, which we leave to future work.

The results on V2 datasets (i.e., the denser and simpler ones) are presented in Appendix~\ref{app:expr}. Briefly, the improvement is relatively smaller than that on V1 datasets because the average entity degree is doubled. However, NeoEA still outperformed all the baselines. Those results empirically proved its effectiveness.

\subsection{The Scalability and Efficiency of NeoEA}
We also conducted experiment on the OpenEA 100K datasets to evaluate performance of NeoEA on larger KGs. We used a single TITAN RTX for training, and SEA (the fastest model) as the basic EEA model. In theory, NeoEA does not have multiple GCN/GAT layers nor the pair-wise similarity estimation on whole graphs. The embedding distributions are also obtained from the sampled KG. Therefore, NeoEA is applicable to large-scale datasets.

The results are shown in Figure \ref{fig:100k}, from which we can find that the time for training one epoch (Figure \ref{fig:100k}a) was not evidently increased, especially considering the cost of switching the optimizers. On the other hand, the overall training time was nearly doubled (Figure \ref{fig:100k}b), caused by the adversarial training procedure. Even thought the loss converged more slowly in NeoEA, we should notice that the overall training time was much less than that of complicated EEA models, e.g., BootEA (35,000+ seconds). In Figures \ref{fig:100k}c and \ref{fig:100k}d, we can find that the improvement for Hits@1 and MRR was significant, especially on the D-Y dataset that took longest time in training.

\subsection{The Necessity of Conditioned Neural Axioms}

We designed an experiment to verify some claims in Section~\ref{sec:Neo}. We choose the best-performing model RDGCN as our baseline. As shown in Table~\ref{tab:ablation_study}, ``Full'' denotes NeoEA with the full set of neural axioms. ``Partial'' denotes NeoEA without the conditional triple axioms. We further removed the conditional entity axioms from ``Partial'' to construct ``Basic'', and the last one, ``Original'', denotes the original EEA model. 

\begin{figure}[t]
	
	\centering
	\hfill
	\begin{subfigure}[h]{0.48\linewidth}
		\includegraphics[width=\textwidth]{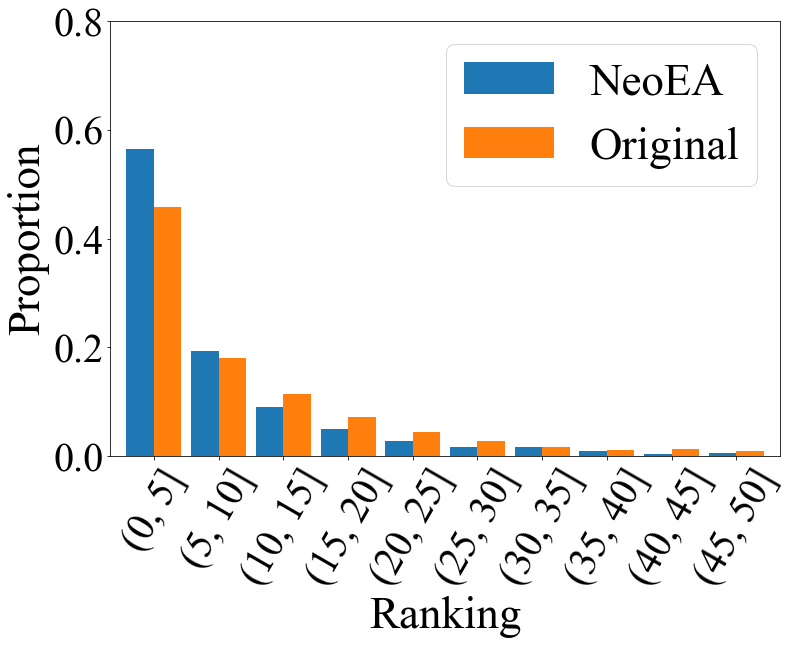}
		\caption{rankings of long-tail entities}\label{fig:long_tail}
	\end{subfigure}
	\hfill
	\begin{subfigure}[h]{0.48\linewidth}
		\includegraphics[width=\textwidth]{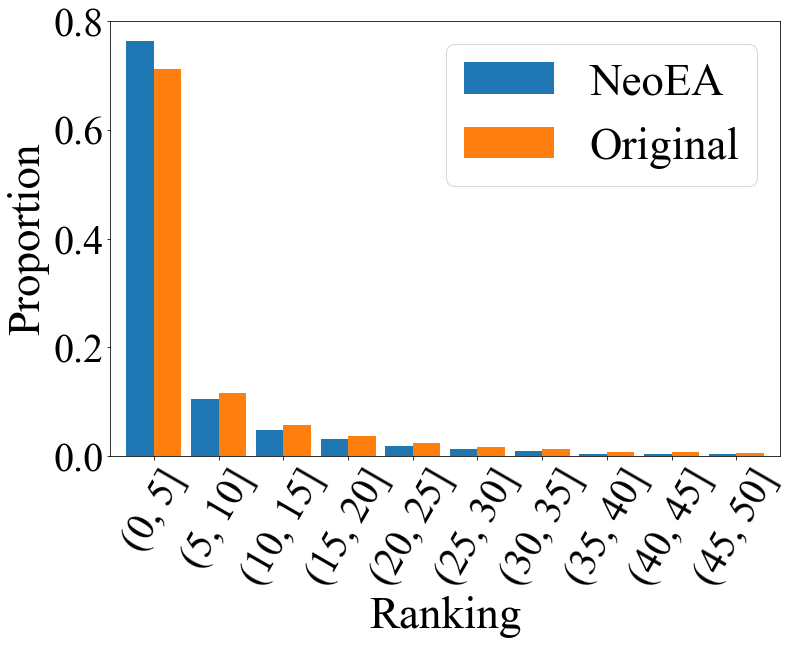}
		\caption{rankings of popular entities}\label{fig:popular}
	\end{subfigure}
	
	\caption{Normalized histograms of alignment rankings on EN-FR, V1 dataset.}
	\label{fig:rank_hist}
\end{figure}

It is clear that aligning the basic axiom was less effective or even harmful to the model. This result empirically demonstrates our assumption that the uniformity property of the learned entity embeddings will make the embedding distribution alignment meaningless. On the other hand, aligning only a part of conditional axioms $\sA_{E_h|r_i}, \sA_{E_t|r_i}$ that describe entity embedding distributions conditioned on relation embeddings was significantly helpful for the model. Also, the additional improvement was observed with the full conditional axioms. 

Note that the improvement from ``Partial'' to ``Full'' was not as significant as that from "Basic" to ``Partial''. This is because the triple neural axioms mainly describe the relationships between different relations (see Appendix~\ref{app:neoea_proof}). Due to the sampling strategy of the datasets, the number of relations is relatively small. 
Few correlated relation pairs exist in the datasets, resulting in limited improvement.

\subsection{Further Analysis of the Bound}

We have shown that the embedding discrepancy between each underlying aligned pair is bounded by $\epsilon$ associated with $\lambda$ in Section~\ref{sec:EEA}. In this section, we provide empirical statistics to verify this point. To this end, we manually split the entities into two groups: (1) long-tail entities, which have at most two neighbors that do not belong to known aligned pairs; (2) popular entities, the remaining. 

We draw the histograms of alignment rankings w.r.t. respective groups based on the EEA model SEA. From Figure~\ref{fig:rank_hist}, we can find that the proportion of the inexact alignments (i.e., ranking $>5$) for long-tail entities is evidently larger than that of popular entities, especially for the bins $(5,20]$. This observation verified that the long-tail entities are less constrained compared to those popular entities in EEA problem. 

Furthermore, with NeoEA, the rankings of those long-tail entities were improved more significantly than those of popular entities, which demonstrates that NeoEA indeed tightened the representation discrepancy of those less restrained entities. 

We report the average ranking improvement on four V1 datasets in Figure~\ref{fig:rank_improvement}, which shows consistent results. It is worth noting that the ranking improvement for long-tail entities is more than twice as larger as that for popular entities, except the D-W dataset.  

\begin{figure}[t]
	\centering
	\includegraphics[width=.98\linewidth]{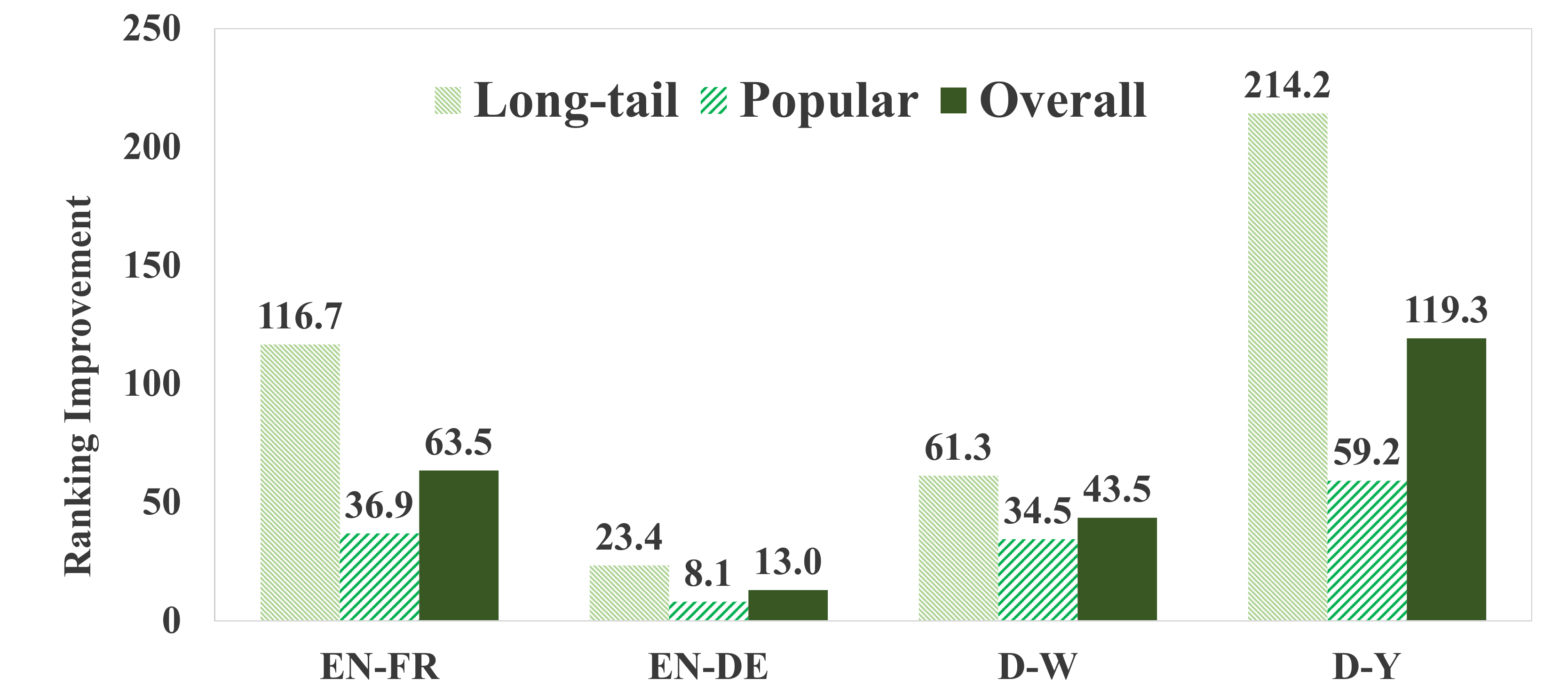}
	\caption{The average ranking improvement on different groups of entities.}
	\label{fig:rank_improvement}
\end{figure}

\section{Conclusion and Future Work}
In this paper, we proposed a new approach to learn KG embeddings for entity alignment. We proved its expressiveness theoretically and demonstrated its efficiency by conducting extensive experiments on the latest benchmarks. We observed that four state-of-the-art EEA methods gained evident benefits with NeoEA. Moreover, we showed that the proposed conditional neural axioms are the key to improve the performance of current EEA methods. For future work, we plan to study how to extend NeoEA with realistic ontology knowledge for further improvement.

\bibliography{references_with_page_number}
\bibliographystyle{ACM-Reference-Format}

\clearpage
\appendix
\section{Proof for Proposition~\ref{proposition:expressiveness}}
\label{app:neoea_proof}
\setcounter{propositionEnd}{2}
\begin{propositionEnd}[Expressiveness]
	\label{propositionEnd:expressiveness}
	Aligning the conditional neural axioms minimizes the embedding discrepancy of two KGs at the ontology level.
\end{propositionEnd}

\begin{proof}
	We split the proofs according to the types of axioms in OWL2 EL \cite{OWL2-EL}:

	\textbf{ObjectPropertyDomain, ObjectPropertyRange.} The proof for ObjectPropertyDomain has been presented in Section~\ref{sec:Neo}, and that for ObjectPropertyRange is similar.

	\textbf{ReflexiveObjectProperty, IrreflexiveObjectProperty.} If we say that a relation $r_i$ is reflexive, it must satisfy
	\begin{align}
		\forall (e, r_i, e^\prime)\in \gT,\ (e, r_i, e) \in \gT,
	\end{align}
	which means each head entities of $r_i$ must be connected by $r_i$ to itself. The above rule suggests that we can align the underlying \emph{reflexive} knowledge by minimizing the discrepancy between triple distributions conditioned on relation $r_i$, i.e., aligning $\sA_{E_{(h,t)}|r_i}^1$ with $\sA_{E_{(h,t)}|r_i}^2$. The similar to IrreflexiveObjectProperty axiom.

	\sloppy \textbf{FunctionalObjectProperty, InverseFunctionalObjectProperty.} We first introduce FunctionalObjectProperty axiom. It compels each head entity $e$ connected by relation $r_i$ to have exactly one tail entity, implying the following rule:
	\begin{align}
		\forall (e, r_i, e^\prime) \in \gT, \forall e''\in \gE,\ (e, r_i, e'') \notin \gT.
	\end{align}
	The above rule is also related to the triple distribution conditioned on $r_i$. The similar to the InverseFunctionalObjectProperty axiom.

	\textbf{SymmetricObjectProperty, AsymmetricObjectProperty.} The first axiom can state a relation $r_i$ is symmetric, that is,
	\begin{align}
		\forall (e, r_i, e^\prime) \in \gT,\ (e^\prime, r_i, e) \in \gT.
	\end{align}
	It is also related to the triple distributions referred to $r_i$, implying that aligning $\sA_{E_{(h,t)}|r_i}^1$ with $\sA_{E_{(h,t)}|r_i}^2$ is sufficient to minimize the difference. The similar to the AsymmetricObjectProperty axiom.

    \textbf{SubObjectPropertyOf, EquivalentObjectProperties, DisjointObjectProperties and InverseObjectProperties.} We show that these axioms also define rules related to triple distributions conditioned on relations. We start from SubObjectPropertyOf, which can state that relation $r_i$ is a sub-property of relation $r_j$ (e.g., ``hasDog'' is one of the sub-properties of ``hasPet''). We formulate it as
    \begin{align}
        \forall (e, r_i, e^\prime) \in \gT,\ (e^\prime, r_j, e) \in \gT.
    \end{align}
    To align the potential SubObjectPropertyOf axioms between two KGs, we can respectively align $(\sA_{E_{(h,t)}|r_i}^1,\sA_{E_{(h,t)}|r_i}^2)$ and $(\sA_{E_{(h,t)}|r_j}^1,\sA_{E_{(h,t)}|r_j}^2)$, such that the joint one $(\sA_{E_{(h,t)}|r_i,r_j}^1,\sA_{E_{(h,t)}|r_i,r_j}^2)$ will also be aligned.
    
    Similarly, if $r_i$ and $r_j$ are equivalent, we can interpret the axiom as
    \begin{align}
    \resizebox{.95\linewidth}{!}{$
        \forall (e, r_i, e^\prime) \in \gT,\ (e^\prime, r_j, e) \in \gT;\quad \forall (e, r_j, e^\prime) \in \gT,\ (e^\prime, r_i, e) \in \gT.
        $}
    \end{align}
    If they are disjoint, the corresponding rule will be
    \begin{align}
        \forall (e, r_i, e^\prime) \in \gT,\ (e, r_j, e^\prime) \notin \gT.
    \end{align}
    If they are inverse to each other, the rule is
    \begin{align}
        \forall (e, r_i, e^\prime) \in \gT,\ (e^\prime, r_j, e) \in \gT.
    \end{align}\\
    
    \textbf{TransitiveObjectProperty.} We show that this axiom is also related to triple distributions conditioned on $r_i$. Supposed that a relation $r_i$ is transitive, then one can derive the following rule:
    \begin{align}
        \forall (e, r_i, e^\prime) \in \gT \& (e^\prime, r_i, e'') \in \gT,\ (e, r_i, e'') \in \gT,
    \end{align}
    which means we can align the potential TransitiveObjectProperty axioms via minimizing the distribution discrepancy between $\sA_{E_{(h,t)}|r_i}^1$ and $\sA_{E_{(h,t)}|r_i}^2$.
\end{proof}

\section{Results on V2 Datasets}
The results on V2 datasets are shown in Table \ref{tab:main_results_v2}. Although entities have doubled number of neighbors in V2 datasets, all baseline models still gained significant improvement with NeoEA.
\label{app:expr}

\begin{table*}[!t]
	\centering
	\caption{Results on V2 datasets (5-fold cross-validation).}
	\resizebox{.82\textwidth}{!}{
		\begin{tabular}{lcccccccccccc}
			\toprule
			\multirow{2}{*}{Models} & \multicolumn{3}{c}{EN-FR} & \multicolumn{3}{c}{EN-DE} & \multicolumn{3}{c}{D-W} & \multicolumn{3}{c}{D-Y}\\
			\cmidrule(lr){2-4} \cmidrule(lr){5-7} \cmidrule(lr){8-10} \cmidrule(lr){11-13}
			& H@1 & H@5 & MRR & H@1 & H@5 & MRR & H@1 & H@5 & MRR & H@1 & H@5 & MRR \\ \midrule
			RSN \cite{RSN}& .579 & .759 & .662 & .791 & .890 & .837 & .723 & .854 & .782 & .933 & .974 & .951 \\
			RSN\,+\,NeoEA & \textbf{.583} & \textbf{.760} & \textbf{.666} & \textbf{.794} & \textbf{.892} & \textbf{.839} & \textbf{.729} & \textbf{.858} & \textbf{.787} & \textbf{.935} & \textbf{.976} & \textbf{.953}\\
			\midrule
			SEA \cite{SEA} & .360 &.651 &.494 & .606 & .779 & .687 & .567 & .770 & .660 & .899 & .950 & .923 \\
			SEA\,+\,NeoEA & \textbf{.375} &\textbf{.666} & \textbf{.508} & \textbf{.637} & \textbf{.800} & \textbf{.712} & \textbf{.588} & \textbf{.784} & \textbf{.677} & \textbf{.917} & \textbf{.959} & \textbf{.936}\\
			\midrule
			BootEA \cite{BootEA} & .660 & .850 & .745 & .833 & .912 & .869 & .821 & \textbf{.926} & .867 & \textbf{.958} & \textbf{.984} & \textbf{.969} \\
			BootEA\,+\,NeoEA & \textbf{.665} & \textbf{.853} & \textbf{.749} & \textbf{.834} & \textbf{.916} & \textbf{.870} & \textbf{.822} & \textbf{.926} & \textbf{.869} & \textbf{.958} & \textbf{.984} & \textbf{.969}\\
			\midrule
			RDGCN \cite{RDGCN}& .847 & .919 & .880 &.833 & .891 & .860 & .623 & .757 & .684 & .936 & .966 & .950 \\
			RDGCN\,+\,NeoEA & \textbf{.864} & \textbf{.933} &\textbf{.896} & \textbf{.849} & \textbf{.902} & \textbf{.874} & \textbf{.632} & \textbf{.760} & \textbf{.690} & \textbf{.940} & \textbf{.970} & \textbf{.953}\\
			\bottomrule
	\end{tabular}}
	\label{tab:main_results_v2}
\end{table*}

\end{document}